\documentclass[usenames,dvipsnames]{article}
\usepackage{amssymb,amsmath}
\usepackage{microtype}
\usepackage{graphicx}
\usepackage{subfigure}
\usepackage{booktabs} 
\usepackage{xspace}
\usepackage{amsthm} 
\usepackage{thmtools}
\usepackage{mathtools}
\usepackage[utf8]{inputenc}
\usepackage{hyperref}
\usepackage{bbm}
\usepackage{csquotes}
\hypersetup{bookmarksopen,bookmarksnumbered,
pdfpagemode=UseOutlines,
colorlinks=true,
}
\usepackage{tikz}
\usetikzlibrary{positioning}

\DeclareMathOperator*{\argmin}{arg\,min}

\newcommand{\argminprob}[1]{\underset{#1}{\argmin}}
\newcommand{\norm}[2]{\left|\left| #1 \right|\right|_{#2}}

\newcommand{\expect}[2]{\mathbb{E}_{#1}\left[#2\right]}

\newcommand{\expectB}[2]{\underset{#1}{\mathbb{E}}\left[#2\right]}

\newcommand{\real}[0]{\mathbb{R}}

\newcommand{\Natural}[0]{\mathbb{N}}
\newcommand{\NaturalPositive}[0]{\Natural^+}

\newcommand{\bbm}{\begin{bmatrix}}
\newcommand{\ebm}{\end{bmatrix}}

\newcommand{\tuple}[1]{\left\langle #1\right\rangle}

\providecommand{\state}[0]{s}
\providecommand{\stateSpace}[0]{\mathcal{S}}
\providecommand{\action}[0]{a}
\providecommand{\actionSpace}[0]{\mathcal{A}}
\providecommand{\transFnDef}[0]{P}

\providecommand{\cost}[0]{C}
\providecommand{\hor}[0]{T}
\providecommand{\perf}[0]{J}

\providecommand{\policy}[0]{\pi}
\providecommand{\policyClass}[0]{\Pi}

\providecommand{\dist}[0]{\rho}
\providecommand{\initDist}[0]{\dist}

\providecommand{\distState}[1]{\dist_{#1}}
\providecommand{\distStateTime}[2]{\dist_{#1}^{#2}}

\providecommand{\loss}[0]{\ell}
\providecommand{\expert}[0]{\policy^*}
\providecommand{\lclass}[0]{\ell^{\mathrm{cs}}}

\providecommand{\learner}[0]{\hat{\policy}}

\providecommand{\data}[0]{\mathcal{D}}
\providecommand{\dataExp}[0]{\mathcal{D}_\mathrm{exp}}

\providecommand{\F}[0]{\mathcal{F}}

\providecommand{\B}[0]{\mathcal{B}}

\providecommand{\ibe}[0]{\epsilon_{\mathrm{be}}}
\providecommand{\lipm}[0]{\ell^{\mathrm{ipm}}}

\providecommand{\simulator}[0]{\Sigma}

\newcommand{\DAgger}[0]{\textsc{DAgger}\xspace}

\newcommand{\algName}[0]{\textsc{ALICE}\xspace}
\newcommand{\FAIL}[0]{\textsc{Fail}\xspace}

\newcommand{\alice}[0]{\textsc{ALICE}\xspace}
\newcommand{\alipm}[0]{\textsc{Fail}\xspace}
\newcommand{\alcov}[0]{\textsc{ALICE-Cov}\xspace}
\newcommand{\alcovipm}[0]{\textsc{ALICE-Cov-Fail}\xspace}

\newcommand{\easy}[0]{\textsc{Easy}\xspace}
\newcommand{\hard}[0]{\textsc{Hard}\xspace}
\newcommand{\goldilocks}[0]{\textsc{Goldilocks}\xspace}
\newcommand{\brake}[0]{\textsc{Brake}\xspace}
\newcommand{\throttle}[0]{\textsc{throttle}\xspace}

\newtheorem{definition}{Definition}[section]

\newtheorem{corollary}{Corollary}[section]

\newtheorem{theorem}{Theorem}

\newtheorem{problem}{Problem}

\newcommand{\xxnote}[3]{}
\ifx\hidenotes\undefined
  \usepackage{xcolor}
  \renewcommand{\xxnote}[3]{\color{#2}{#1: #3}}
\fi

\usepackage[accepted]{icml2020}
\icmltitlerunning{Feedback in Imitation Learning}
\begin{document}

\twocolumn[
\icmltitle{Feedback in Imitation Learning: \\ The Three Regimes of Covariate Shift}
\icmlsetsymbol{equal}{*}
\begin{icmlauthorlist}
\icmlauthor{Jonathan Spencer}{princeton}
\icmlauthor{Sanjiban Choudhury}{aurora}
\icmlauthor{Arun Venkatraman}{aurora}
\icmlauthor{Brian Ziebart}{aurora}
\icmlauthor{J. Andrew Bagnell}{aurora,cmu}
\end{icmlauthorlist}
\icmlcorrespondingauthor{S. Choudhury}{schoudhury@aurora.tech}
\icmlaffiliation{princeton}{Princeton}
\icmlaffiliation{aurora}{Aurora Innovation}
\icmlaffiliation{cmu}{Carnegie Mellon}
\icmlkeywords{Imitation learning}
\vskip 0.3in
]
\printAffiliationsAndNotice{} 

\begin{abstract}
Imitation learning practitioners have often noted that conditioning policies on previous actions leads to a dramatic divergence between ``held out'' error and performance of the learner \emph{in situ}. Interactive approaches \cite{ross2011reduction} can provably address this divergence
but require repeated querying of a demonstrator. 
Recent work identifies this divergence as stemming from a ``causal confound''~\cite{de2019causal,wen2020fighting} in predicting the current action, and seek to ablate causal aspects of current state using tools from causal inference. In this work, we argue instead that this divergence is simply another manifestation of \textit{covariate shift}, exacerbated particularly by settings of \textit{feedback} between decisions and input features. The learner often comes to rely on features that are strongly predictive of decisions, but are subject to strong covariate shift.

Our work demonstrates a broad class of problems where this shift can be mitigated, both theoretically and practically, by taking advantage of a simulator but \textit{without} any further querying of expert demonstration. We analyze existing benchmarks used to test imitation learning approaches and find that these benchmarks are
realizable and simple and thus insufficient for capturing the harder regimes of error compounding seen in real-world decision making problems. We find, in a surprising contrast with previous literature, but consistent with our theory, that naive behavioral cloning provides excellent results. We detail the need for new standardized benchmarks that capture the phenomena seen in robotics problems.
\end{abstract}

\section{Introduction}
\label{sec:introduction}

\begin{figure}
    \centering
    \includegraphics[width=\linewidth]{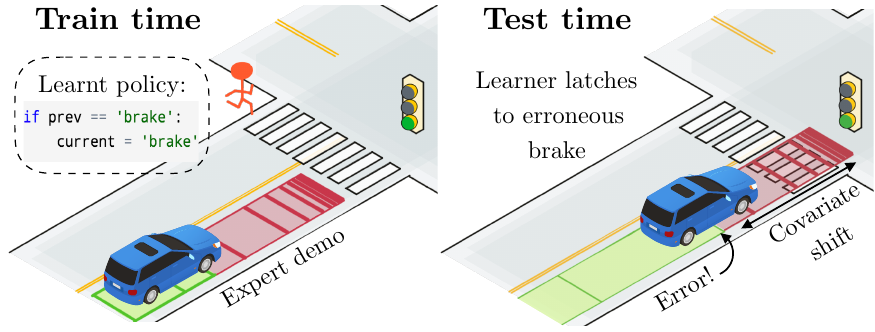}
    \caption{A common example of feedback-driven covariate shift in self-driving. At train time, the robot learns that the previous action (\brake) accurately predicts the current action almost all the time. At test time, when the learner mistakenly chooses to \brake, it continues to choose \brake, creating a bad feedback cycle that causes it to diverge from the expert.}
    \label{fig:covariate_shift}
\end{figure}

Imitation learning (IL) is a method of training robot controllers using demonstrations from human experts. In contrast to, e.g., reinforcement learning (RL), IL controllers can learn to perform complex control tasks ``offline'', without interaction in the environment, and achieve remarkable proficiency with relatively few demonstrations. This speedup comes at a cost, however, since offline learning in robotic control introduces feedback artifacts which must be accounted for.

Consider the problem of training a a self-driving car to stop at a yellow light, as depicted in Figure~\ref{fig:covariate_shift}. When an expert human driver decides whether to \throttle or \brake, they may pay attention to timing of the light, pedestrians, or their previous action. A robot trained to naively imitate this demonstration, however, will latch onto the previous action as it predicts the current action almost all the time. At test time, when the learner inevitably makes a mistake, this mistake feeds into the features for the next timestep. This starts a vicious cycle of errors, as the learner diverges from the expert, not having learned how to recover. 

This driving scenario is a classic example of a \textit{feedback loop}~\cite{bagnell2016feedback,sculley2014machine} between the learners actions and the input features a learner sees. This effects results in a \textit{covariate shift} between training data and the input data a learner encounters when it executes its own policy. This shift leads to often a dramatic divergence between ``hold out`` error and performance of the learner actually executing the policy \textit{in situ}. Feedback-driven covariate shift in the context of imitation learning was first observed by~\cite{Pomerleau-1989-15721} and has recently resurfaced as a often recurring problem in self-driving~\cite{kuefler2017imitating,bansal2018chauffeurnet,codevilla2019exploring}. 
Imitation learner error in settings with explicit feedback is sometimes attributed to a causal confound \cite{de2019causal}, however, we argue that it is simply a more visible case of feedback-driven covariate shift.

A common approach to address covariate shift in imitation learning is to interactively query an expert in states that the learner visits~\cite{ross2011reduction, rossthesis} using the \DAgger algorithm. Consider a class of learner policies $\policy$ that can achieve a one-time-step classification error of $\epsilon$ with respect to the expert's chosen action. Whereas naive behavioral cloning (BC) suffers a classification regret of $O( T^2 \epsilon)$ that grows quadratically in horizon $T$, interactive imitation learning approaches like \DAgger guarantee the optimal $O(T \epsilon)$ regret even in the worst case. However, interactively querying experts is impractical in many cases where human experts are unable to provide feedback online.

Although interactive methods are provably  necessary in the worst case, we posit that there are many problems with more modest demands in a ``Goldilocks Regime'' of moderate difficulty that can be solved \textit{without} the need for an interactive expert. We propose solving those problems in a modified imitation learning setting where we have access to a fixed set of cached expert demonstrations as well as generative access to a high fidelity simulator of the environment. This is similar to inverse reinforcement learning settings that benefit from environment/simulator access during the process of training \cite{abbeel2004apprenticeship,ziebart2008maximum,ho2016generative,finn2016guided} In this setting, we introduce a novel class of algorithms that \textit{provably} address feedback-driven covariate shift.

Our key insight is that, as long as the expert demonstrator visits all states that the learner will visit, then the density ratio between the expert and learner is bounded, $P_\textrm{learner}(\state)/P_\textrm{expert}(\state)\leq C$, we can \textit{correct} for the covariate shift and enjoy a regret of just $O(\epsilon T )$ using just a simulator and cached demonstrations.

Our contributions are the following:
\begin{enumerate}
    \item We introduce a general framework for mitigating covariate shift using cached expert demonstrations, called \algName (Aggregate Losses to Imitate Cached Experts)
    \item We propose a family of loss functions within that framework and analyze their consistency.
    \item We identify specific characteristics needed for better benchmarks in imitation learning.
\end{enumerate}

\section{Preliminaries: The Feedback Problem}
\label{sec:prob_form} 

For simplicity, we consider an episodic finite-horizon Markov Decision Process (MDP) $\tuple{\stateSpace, \actionSpace, \cost, \transFnDef, \initDist, \hor}$ where $\stateSpace$ is a set of states, $\actionSpace$ is a set of actions, $\cost: \stateSpace \rightarrow [0,1]$ is a (potentially unobserved) state-dependent cost function, $\transFnDef: \stateSpace \times \actionSpace \rightarrow \Delta(\stateSpace)$ is the transition dynamics, $\initDist \in \Delta(\stateSpace)$ is the initial distribution over states, and $\hor \in \NaturalPositive$ is the time horizon.
Given a policy $\policy: \stateSpace \rightarrow \Delta(\actionSpace)$, let $\distStateTime{\policy}{t}$ denote the distribution over states at time $t$ following $\policy$. Let $\distState{\policy}$ be the average state distribution $\distState{\policy} = \frac{1}{\hor} \sum_{t=1}^{\hor} \distStateTime{\policy}{t}$. Let $Q^{\policy}$ be the cost-to-go for selecting $\action$ at $\state$ and following $\policy$ thereafter, $Q^{\policy}_t(\state,\action)=\cost(\state)+\expect{\state\sim\transFnDef(\state,\action)}{\cost(\state')}+\sum_{t'=t+1}^{T-1}\expect{\state''\sim\transFnDef(\state_{t'},\policy(\state_{t'}))}{\cost(\state'')}$. Let the (dis)-advantage function $A^{\policy}$ be the difference in cost-to-go between selecting action $\action$ and selecting action $\policy(\action)$, $A^{\policy}(\state,\action)=Q^{\policy}(\state,\action)-Q^{\policy}(\state,\policy(\state))$.
The total cost of executing policy $\policy$ for $\hor$-steps is $\perf(\policy) = \sum_{t=1}^{\hor} \expect{\state_t \sim \distStateTime{\policy}{t}}{\cost(\state_t)} = \hor \expect{\state \sim \distState{\policy}}{\cost(\state)}$.

In imitation learning, we do not observe $\cost(\state)$. Instead, we observe only \emph{cached expert demonstrations} $\dataExp = \{(\state_t^*, \action_t^*)\}$ generated \emph{a priori} by the expert $\expert$, i.e, $\state_t^* \sim \distStateTime{\expert}{t}, \action_t^* \sim \expert(.|\state_t^*)$. The cost $\cost(\state)$ can be constructed as a ``mismatch'' loss relative to the expert (i.e. 0-1 classification loss for discrete actions or squared error in state or action for continuous spaces). Whether $\cost(\state)$ is set by the MDP or relative to the expert, the analysis is the same, and the goal is to train a policy $\policy$ that minimizes the performance difference $\perf(\policy) - \perf(\expert)$ via imitation. When $\cost(\state)$ is a loss relative to the expert, minimizing performance difference is equivalent to minimizing on-policy expert mismatch.

\begin{equation*}
\perf(\policy)-\perf(\expert) = \hor \expect{\state \sim \distState{\policy}}{\mathbbm{1}_{\policy(\state)\neq\expert(\state)}}
\end{equation*}

\paragraph{Feedback Drives Covariate Shift} 

The traditional approach to imitation learning, Behavior Cloning (BC)~\cite{Pomerleau-1989-15721}, simply trains a policy $\policy$ that correctly classifies the expert actions on a cached set of demonstrations: 
\begin{equation}
\label{eq:bc}
\learner = \argminprob{\policy \in \policyClass}\; \expect{ (\state^*,\action^*) \sim \dataExp}{\lclass(\policy(\state^*), \action^*)}. 	
\end{equation}
where $\lclass$ is a classification loss. Here, the best we could hope for is to drive down the training error $\expect{ (\state^*,\action^*) \sim \dataExp}{\lclass(\learner(\state^*), \action^*)} \leq \epsilon$ and ensure a similar validation error on held out data. This is perfectly reasonable in the supervised learning setting where low hold-out error implies low test error.

However, sequential decision making tasks contain inherent \emph{feedback}, where past action $\action_{t-1}$ affect the input to our learner's decision making at time $t$. As shown in the graphical model in Fig.~\ref{fig:feedback}, this can happen through multiple channels. It can happen indirectly through the MDP dynamics where $\action_{t-1}$ affects $\state_t$, and can also happen explicitly when $\policy$ is directly conditioned on previous actions, i.e.  $\policy(\action_t|\state_t,\action_{t-1})$. In fact, dependence on past actions is often essential to ensure model smoothness or hysteresis. 

\begin{figure}
    \centering
    \input{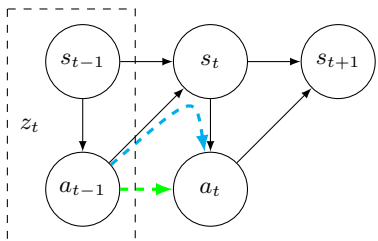}
\begin{tikzpicture}[auto,>=latex,font=\small]

    \tikzset{
    node distance=7mm,
    state/.style={minimum size=28pt,font=\small,circle,draw},
    edge/.style={->},
    }
    \node[state] (st) {$\state_{t}$};
    \node[state,left=of st] (stm1) {$\state_{t-1}$};
    \node[state,right=of st] (stp1) {$\state_{t+1}$};
    \node[state,below=of stm1] (atm1) {$\action_{t-1}$};
    \node[state,below=of st] (at) {$\action_{t}$};

    \draw[edge] (stm1) -- (st);
    \draw[edge] (atm1) -- (st);
    \draw[edge] (st) -- (stp1);
    \draw[edge] (at) -- (stp1);
    \draw[edge] (stm1) -- (atm1);
    \draw[edge] (st) -- (at);
    \draw[cyan,dashed,very thick,->] (atm1) .. controls(-.2,-.4) .. (at);
    \draw[green,dashed,very thick,->] (atm1) -- (at);
    
    \draw[dashed] (-2.7,.7) rectangle (-1,-2.4);
    \node[font=\small] (zt) at (-2.4,-.85) {$z_t$};

\end{tikzpicture}
\vspace{-.5cm}
    \caption{Inherent feedback in sequential decision making tasks. Past action $\action_{t-1}$ affects current action $\action_t$, either indirectly (blue) via MDP dynamics or directly (green) via explicit conditioning.}
    \label{fig:feedback}
\end{figure}

A commonly observed phenomena~\cite{Pomerleau-1989-15721, ross2010efficient,brantley2019disagreement} is that such feedback loops drive \emph{covariate shift}, where the states experienced by the learner $\state_t \sim \distStateTime{\policy}{t}$ diverge significantly from $\dataExp$. Although a learner begins by making an error of just $\epsilon$, those errors drive the learner into novel states rarely demonstrated by the expert where it is even more likely to make mistakes. Distributional shift and compounding errors result in total error $O(T^2 \epsilon)$.  This is formalized below: 
\begin{theorem}[Theorem 2.1 in \cite{ross2011reduction}]
Let $\expect{\state \sim \distState{\pi}}{\lclass(\learner(\state), \expert(\state))} \leq \epsilon$ be the bounded on-policy training error, where $\lclass$ is the 0-1 loss (or an upper bound). We have $\perf(\learner) \leq \perf(\expert) + \hor^2 \epsilon$
\end{theorem}

\cite{ross2011reduction} go on to show that it is possible to achieve the ideal $O(T \epsilon)$\footnote{We are usually interested in a mismatch cost $C(s)=\mathbbm{1}_{\policy(\state)\neq\expert(\state)}$, which in Theorem~\ref{thm:bc_medium_ross} implies $u = 1$, giving us this $O(T \epsilon)$ bound.}  error by interactively querying the expert on states visited by the learner $\state \sim \distState{\pi}$, and minimizing  the same classification loss $\lclass$.

\begin{theorem}[Theorem 2.2 in \cite{ross2011reduction}]
\label{thm:bc_medium_ross}
Let $\expect{\state \sim \distState{\pi}}{\lclass(\learner(\state), \expert(\state))} \leq \epsilon$ be the bounded on-policy training error, and $A^{\expert}(\state, \action) \leq u$ be the bounded (dis)-advantage w.r.t expert $\forall \state, \action$. We have $\perf(\learner) \leq \perf(\expert) + u \hor \epsilon$.
\end{theorem}

However, querying the expert online is impractical in many settings \cite{laskey2017comparing}, requiring a human to provide cognitively challenging low-latency open-loop controls. This inspires a search for methods which achieve similar bounds \textit{without} the need for an interactive demonstrator, asking the following question:

\begin{problem}
\label{prob:cached}
Can we achieve $O(T \epsilon)$ error without querying the expert online, i.e., using only the cached expert $\dataExp$?
\end{problem}

\section{Confusion on Causality and Covariate Shift}
\label{sec:causal}

We now take a slight detour to review instances where feedback causes real problems in imitation learning for self-driving, and clarify a widespread confusion between covariate shift and causality. A common observation in all these instances is that while past actions are strongly correlated with future actions, it often leads to a latching effect where once the robot begins to brake, it continues braking. ~\cite{muller2006off} first note such correlations this with steering actions, but more recently~\cite{kuefler2017imitating, bansal2018chauffeurnet,de2019causal} note this with braking actions and ~\cite{codevilla2019exploring} note a correlation between speed and desired acceleration.

In an important paper,~\cite{de2019causal} look to understand the problem of feedback through the lens of causal reasoning \cite{pearl2016causal, spirtes2000causation}. They note that, practically, the most severe repercussions come from features that enable the learner to condition almost directly on previous actions.\footnote{It is difficult, however, to formalize this intuition because in any interesting imitation learning problem, actions do indeed affect the states the learner must operate with.}
The authors propose that:
\textquote{This situation presents a give-away symptom of causal misidentification: access to more information leads to worse generalization performance in the presence of distributional shift. Causal misidentification occurs commonly in natural imitation learning settings, especially when the imitator’s inputs include history information.}
and that their contribution is:
\textquote{We propose a new interventional causal inference approach suited to imitation learning.}

Fig.~\ref{fig:feedback} shows the graphical model under consideration.  The authors claim $z_t = [\state_{t-1}, \action_{t-1}]$ is a confounder between the independent variable $\state_t$ and the dependent variable $\action_t$. However, $z_t$ is fully observed. From ~\cite{pearl2016causal} (Theorem 3.2.2), $\state_t$ and $\action_t$ are not confounded iff $P( \action_t | \mathrm{do} (\state_t) ) = P(\action_t | \state_t)$, i.e. whenever the observationally witnessed association between them is the same as the association that would be measured in a controlled experiment, with $\state_t$ randomized. In the experimental setups they considered, despite explicit conditioning on previous actions we failed to find true confounders as all variables are observed.

\begin{figure*}[ht]
    \centering
    \includegraphics[width=0.8\linewidth]{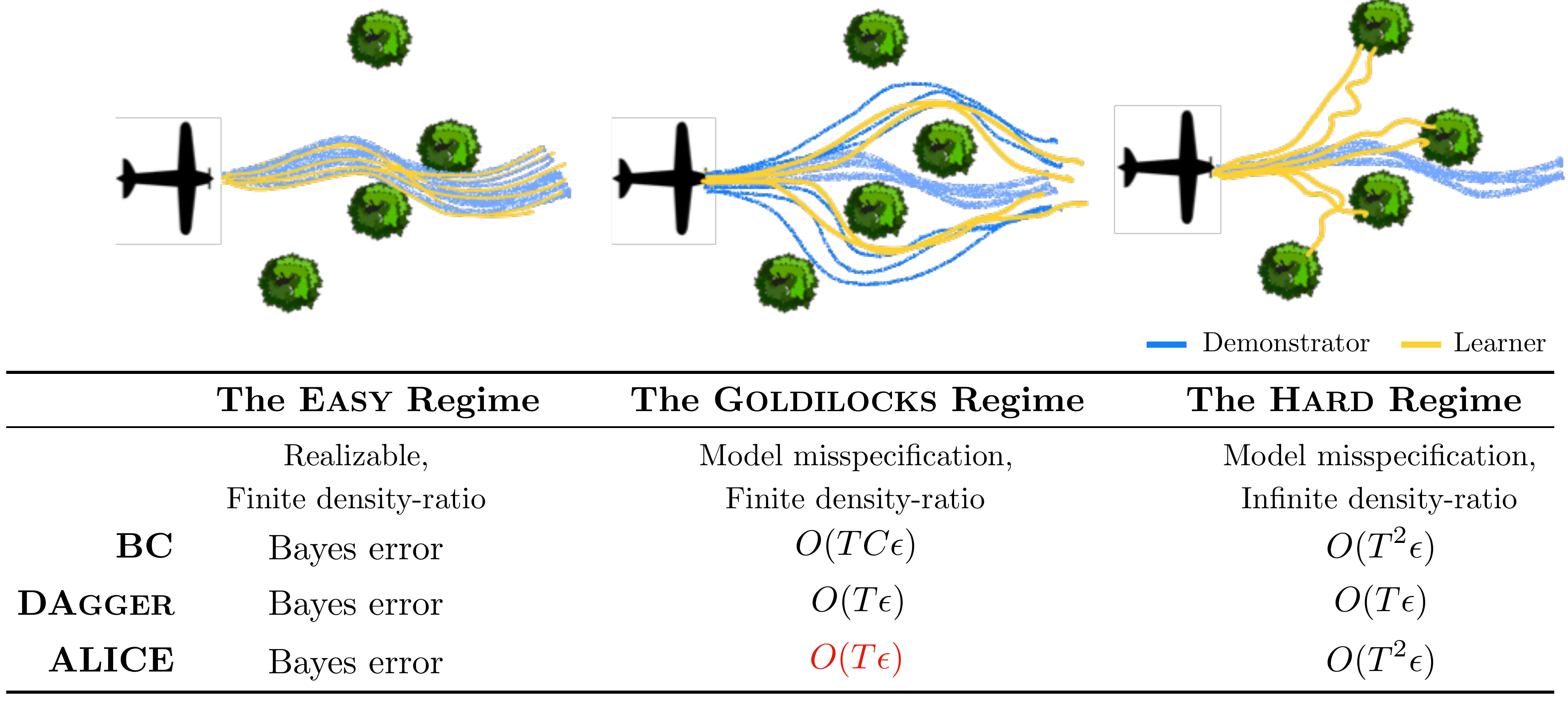}
    \caption{Spectrum of feedback driven covariate shift regimes. Consider the case of training a UAV to fly through a forest using demonstrations (blue). In \easy regime, the demonstrator is realizable, while in the \goldilocks and \hard regime, the learner (yellow) is confined to a more restrictive policy class. While model mispecification usually requires interactive demonstrations, in the \goldilocks regime, \alice achieves $O(T\epsilon)$ without interactive query.
    }
    \label{fig:spectrum}
\end{figure*}

When all variables are observed, under infinite data regime and full support (all modes excited), behavior cloning / MLE is consistent. Indeed \cite{zhang2020causal} similarly prove that \textquote{if the expert and learner share the same policy space, then the policy is always imitable}. While the authors of \cite{de2019causal} claim \textquote{We define causal misidentification as the phenomenon whereby cloned policies fail by misidentifying the causes of expert actions}, we posit that this is not a causal issue but rather one of finite data regime and/or model misspecification. The learner prediction error $\epsilon$ concentrates unevenly as the learner comes to rely on a features that are strongly predictive of actions, but create strong feedback loops that lead to covariate shift.

Finally, if we alter the graphical model to have an unobserved casual confounders,~\cite{de2019causal} is correct to note that this would manifest as a covariate shift and lead to poor performance even in the infinite data limit. However, their proposed approach requires expert intervention similar to \DAgger (although is can be much more sample efficient), which makes it difficult to apply in our setting. Instead, we propose to measure the induced covariate shift and directly adjust for it.

\section{Feedback Driven Covariate Shift Spectrum}

We return to Problem~\ref{prob:cached} -- achieve $O(T\epsilon)$ error without repeatedly querying the expert. We adopt the large-sample, \textit{reduction}-style analysis of \cite{beygelzimer2008machine, ross2010efficient, ross2011reduction} and for simplicity consider the infinite data regime to keep the focus squarely on understanding the error-compounding problem so common in real-world applications (Section~\ref{sec:causal}). 

While the theoretical analysis demonstrates that BC must incur $O(T^2 \epsilon)$ error in the worst-case, practitioners have often noted settings where behavior cloning is effective~\cite{bojarski2016end}. This suggests that there maybe varying regimes of difficulty in the imitation learning problem that call for different algorithmic assumptions and approaches as shown in Fig.~\ref{fig:spectrum}

\subsection{The \easy Regime: Realizable}
In the simplest case, the expert is realizable $\expert \in \policyClass$. We can simply drive the classification error to $\epsilon = 0$, recovering the expert policy. The prescription in this regime is the easiest practically and often quite powerful: collect as much data as possible and ensure realizability with a sufficiently broad policy class.

Crucially, we note theoretically and demonstrate below empirically that the benchmarks used in recent imitation learning literature \cite{ho2016generative,barde2020adversarial} all fall into the \easy regime. Without restrictions on model class, we find that the commonly used benchmarks are well solved by naive behavioral cloning and fail to capture real-world issues \cite{sculley2014machine, Pomerleau-1989-15721, kuefler2017imitating, bansal2018chauffeurnet, codevilla2019exploring}. We suggest the methodological issues that may have led to prior beliefs that these benchmarks captured error-compounding and suggest improvements for future benchmarks.

\subsection{The \hard Regime: Model Misspecification, Infinite Density-ratio} This is the hardest case, where the expert is not realizable, $\expert \notin \policyClass$, and does not visit every state, thus leading to infinite learner-expert density ratio. As \cite{ross2011reduction} note, when a fundamental limit (of $\epsilon$ classification error) is imposed on how accurately we can learn the expert's policy whether due to partial observability, optimization limitations, or regularization, a compounding error of $O(T^2 \epsilon)$ is simply inevitable without repeated access and interaction with the demonstrator. Interaction with the demonstrators enables the learner to achieve the best possible $O( T \epsilon)$ error. This comes fundamentally from the fact that there may simply be states that are \emph{never} visited under the demonstrator policy and thus there is effectively no way to recover if an error leads our learner to such states. As such, interactive methods like \DAgger \cite{ross2011reduction} are inevitably required to achieve high performance.

\subsection{The  \goldilocks  Regime: Model Misspecification, Finite Density-ratio} 
Consider now a regime where the expert is not realizable, $\expert \notin \policyClass$, and there thus exists some inevitable classification error (captured crudely via the $\epsilon$ bound),
\textit{but} for which there is \emph{sufficient exploration} in the data collected by observing the demonstrator that we can hope to learn good ``recovery'' behavior. We quantify that notion of exploration by saying that the demonstrator visits all states with some small probability, thus bounding the learner-demonstrator density ratio $\norm{{\distState{\learner}}/{\distState{\expert}}}{\infty} \leq C$. We denote this the \textit{Goldilocks} regime as we speculate this model is ``just right'' to capture much of the error compounding difficulty seen in real-world imitation learning.

We begin by noting that due to misspecification, a naive BC can still only provide sub-optimal performance $O(TC\epsilon)$.

\begin{theorem}[BC in Goldilocks regime]
\label{thm:bc_medium}
Let $\learner$ be the learnt policy such that $\expect{ (\state^*,\action^*) \sim \dataExp}{\lclass(\learner(\state^*), \action^*)} \leq \epsilon
$. For bounded density ratio $ \norm{{\distState{\learner}}/{\distState{\expert}}}{\infty} \leq C$ and bounded (dis)-advantage $A^{\expert}(\state, \action) \leq u \; \forall (s,a)$, we have
\begin{equation}
    J(\learner) \leq J(\expert) + \min(T C u \epsilon, T^2 \epsilon)
\end{equation}
\end{theorem}
\begin{proof}
See appendix \ref{proof:bc}
\end{proof}

In this regime, the problem becomes one of deploying finite model resources to lead to as good as possible performance.
We approach Problem~\ref{prob:cached} from a perspective of covariate-shift correction and provide, for the first time to our knowledge, an approach that achieves $O(T \epsilon)$ without interactively querying the expert.

\section{Approach}

We speculate that the \goldilocks regime is relatively common in real-world applications and that in the large data limit the need to re-query an expert may be relatively uncommon.
Although an interactive expert is sufficient to address this issue, we believe that only a set of \textit{cached expert demonstrations} is necessary. We thus consider the use of classical \textit{covariate-shift} mitigation strategies, though the IL setting introduces a complication. Due to feedback we have a subtle chicken-or-egg problem: high-performance covariate shift mitigation strategies rely on samples from the distribution of inputs (states) from the test distribution. However, the test distribution here is \textit{directly} an effect of decisions made earlier by the policy.

We show how this can be addressed both for non-stationary and for stationary policies. For non-stationary policies (Sec~\ref{sec:approach:forward}), we use sequential forward training of successive policies. For stationary policies (Sec~\ref{sec:approach:iterative}), we use an iterative approach which refines both policy and mitigation strategy until they stabilize.

We present \alice (Aggregate Losses to Imitate Cached Experts), a family of algorithms that train a policy to imitate cached expert while countering or adjusting for covariate shift. Abstractly, all variants of \alice optimize:
\begin{equation}
\label{eq:alice:absract}
\learner = \argminprob{\policy \in \policyClass}\;  \loss(\policy, \dataExp, \Sigma)  	
\end{equation}

where the loss $\loss(.)$ uses both cached demonstrations $\dataExp$ and a simulator\footnote{This demonstration + simulator requirement is equivalent to the setting of, \textit{e.g.} GAIL \cite{ho2016generative}, where the learner has access to the environment but does \textbf{not} require access again to the demonstrator after the initial collection.}
$\simulator: \policy \mapsto \distState{\policy}(s)$ to build the induced density $\distState{\learner}(s)$. There exist several choices for the loss $\loss(.)$ of which we analyze three, and show that they all lead to an $O(T \epsilon)$ bound under varying assumptions.

\subsection{Forward Training Non-Stationary Policies}
\label{sec:approach:forward}
A simple resolution to the chicken-or-egg problem in \eqref{eq:alice:absract} is to train a different policy $\learner_t$ for every timestep. This leads to a simple, recursive recipe: at each time $t$, roll-in trained policies $\learner_1, \dots, \learner_{t-1}$ in the simulator $\Sigma$ to construct $\distStateTime{\learner}{t}(s)$ and solve
\begin{equation}
\label{eq:onpolicy}
\learner_{t} = \argminprob{\policy \in \policyClass}\;  \loss_t(\policy, \dataExp, \distStateTime{\learner}{t})  	
\end{equation}

While training non-stationary policies is at times impractical, we use this setting to theoretically analyze various loss functions $\loss(.)$, deferring to Section~\ref{sec:approach:iterative} for a simple online variant with no fixed horizon.

\subsection{Family of ALICE Loss Functions}
\label{sec:approach:loss}
We analyze two different losses -- reweigh expert demonstrations by estimating the learner-expert density ratio, and match next state moments of learner-expert distribution. We also analyze a third loss that combines both ideas to be doubly robust. 

\textbf{Loss 1 (\alcov): Reweigh by Density Ratio}

A classic approach to covariate shift mitigation \cite{shimodaira2000improving} is to reweigh samples by the density ratio of the target (learner) to the observed (expert) distribution. 
\begin{equation}
\begin{split}
\label{eq:alcov}
\loss_t(\policy, \dataExp,& \distStateTime{\learner}{t}) = \\ &\expectB{\state^*_t, \action^*_t \sim \dataExp}{ \left(\frac{\distStateTime{\learner}{t}(s^*_t)}{\distStateTime{\expert}{t}(s^*_t)} \right) \lclass(\pi(\state_t^*), \action_t^*) ) }
\end{split}
\end{equation}

where $\lclass(.)$ is the plain classification loss. In practice, the density ratio is not known and must be estimated $\hat{r}(s) \approx \left(\frac{\distStateTime{\learner}{t}(s^*_t)}{\distStateTime{\expert}{t}(s^*_t)} \right)$\footnote{In practice, since the density ratio cannot be estimated perfectly, we often apply an exponential weighting $\hat{r}(s_t)^\alpha$, $\alpha \in[0,1]$ which brings the estimate closer to uniform weighting and reins in unreasonably large values, tuning $\alpha$ through validation.}. \alcov is the simplest variant that requires only that the ratio estimate be bounded to guarantee $O(T \epsilon)$ 

\begin{theorem}[\alcov]
\label{thm:alcov}
Let $\learner$ be the learned policy produced by \alcov such that $\loss_t(\learner, \dataExp, \distStateTime{\learner}{t}) \leq \epsilon$. Let $\hat{r}(s_t)$ be the density ratio estimate such that $ \expect{\state^*_t \sim \dataExp}{ \hat{r}(s_t) - \frac{\distStateTime{\learner}{t}(s_t)}{\distStateTime{\expert}{t}(s_t)}} \leq \gamma $. Let $A^{\expert}(\state, \action) \leq u \; \forall (s,a)$ be a bound on the (dis)-advantage w.r.t expert. We have
\begin{equation}
    J(\learner) \leq J(\expert) + T u (\epsilon + \gamma)
\end{equation}
\end{theorem}
\begin{proof}
See appendix ~\ref{proof:alcov}
\end{proof}
\begin{corollary}
Let $\cost(\state)$ be the classification loss $\lclass(\pi(s), \expert(s))$ with respect to expert and $\learner$ be a policy produced as in Theorem.~\ref{thm:alcov}. In this case
\begin{equation}
    J(\learner) \leq J(\expert) + T (\epsilon + \gamma)
\end{equation}
\begin{proof}
For $\cost(\state) = \lclass(\pi(s), \expert(s))$, we have $u=1$.
\end{proof}
\end{corollary}

We incur a penalty $u$ in Theorem.~\ref{thm:alcov} because we are matching expert actions via $\lclass(.)$ instead of matching values, \textit{i.e.} minimizing the expert dis-advantage function. We look at how to do this next.

\textbf{Loss 2 (\alipm): Match next-state moments}

Since our ultimate objective is to match the expert's value function, we can perhaps do so more effectively by matching moments of that function rather than matching individual actions. The loss function introduced in the \FAIL algorithm \cite{Sun-2019-118906} does exactly that, matching moments of the expert and learner's next state distribution $\rho_{t+1}$ for non-stationary policies. We reintroduce that loss here within our general framework\footnote{The \FAIL setting is observation only and non-stationary. Since in our setting we observe expert actions, the adversary function class can be expanded to $\F:\stateSpace\times\actionSpace\rightarrow\real$, searching instead for the state-action value function. This requires a stationary target optimization policy $\policy$, and that we uniformly sample both $\action_t$ and $\action_{t+1}$ and apply appropriate reweighting}, and combine it with covariate shift correction in the next section. Please refer to \cite{Sun-2019-118906} for the full implementation details of \alipm.

Assuming value functions $V^*_{t+1}(s)$ belong to a class of moments $\F_{t+1} : \stateSpace \rightarrow \real$, the learner can try to match the expert's moments of the next state. An example of a metric that captures deviation from matching moments is the integral probability metric, where an adversary searches for a moment function $f \in \F_{t+1}$ that maximizes learner-expert moment mismatch. Given a learner distribution $\rho_t$ and expert distribution $\rho^*_{t+1}$, we can define the following IPM metric:
\begin{equation}
\begin{aligned}
d(\policy | \rho_t, \rho^*_{t+1}) = \max_{f \in \F_{t+1}} &
\expectB{\substack{ \state_t \sim \rho_t \\ \action_t \sim \policy \\ \state_{t+1} \sim P(.|\state_t, \action_t) }}{f(\state_{t+1})} \\
 &- \expectB{s^*_{t+1} \sim \rho^*_{t+1}}{ f(\state^*_{t+1})}
\end{aligned}
\end{equation}
where the first term is the the expected moment on states obtained by rolling out $\pi$ and the second term is the expected expert moment. \alipm defines the loss as:
\begin{equation}
\begin{aligned}
\loss_t(\policy, \dataExp, \distStateTime{\learner}{t}) = d(\policy | \distStateTime{\learner}{t}, \dataExp)
\end{aligned}
\end{equation}

While \alipm does \textit{not} require a bounded density ratio to remove error compounding, it does require a different, at times stronger condition -- the notion of one-step recoverability
\begin{definition}[One-step Recoverability]
\label{def:recoverability}
For any state distribution $\rho_t$, there exists a policy $\pi$ that bounds $d(\pi | \rho, \dataExp) \leq \epsilon$.
\end{definition}
This basically requires that no matter what the current distribution, the learner can recover in a single time-step to drive the loss to $\epsilon$. Without this condition, divergences at earlier time-step could be unrecoverable and hence compound over time leading to $O(T^2 \epsilon)$.
When considering any IL method where $\cost(\state)$ and corresponding value function $V^*(s)$ are inherent to the MDP (rather than relative to the expert as in $\cost(\state)=\lclass(\pi,\state)$), recoverability becomes a fundamental requirement. If the environment includes hard branches between high and low reward paths, bounds on lost reward ($\epsilon$ here, $u$ in \DAgger) becomes arbitrarily large since optimal actions are useless after an initial mistake. These requirements of recoverability can be thought of as imposing a limit on the reward ``branchiness'' of the environment, which we require here only in expectation. This is illustrated in Appendix~\ref{sec:recoverability}. With this condition, we have:

\begin{theorem}[\alipm ]
\label{thm:alipm}
Let $\learner$ be the learnt policy produced by \alipm such that $\loss_t(\learner, \dataExp, \distStateTime{\learner}{t}) \leq \epsilon$, assuming the expert is one-step recoverable, and $\ibe$ be the Inherent Bellman Error (IBE). We have:
\begin{equation}
    J(\learner) \leq J(\expert) + 2 T (\epsilon + \ibe)
\end{equation}
\end{theorem}
\begin{proof}
We require small IBE to account for the richness of the chosen function class $\F$ and its ability to realize $V^*$. See appendix ~\ref{proof:alipm} for proof and definition of IBE.
\end{proof}

\textbf{Loss 3 (\alcovipm): Doubly robust loss}

When we combine Loss 1 and 2, we still require a version of recoverability (captured completely in the loss bound $\epsilon$), but we can eliminate dependence on dis-advantage and IBE. We consider the following loss:
\begin{equation}
\begin{split}
 \loss_t(\policy, & \dataExp, \distStateTime{\learner}{t})= \\ &\expectB{\state^*_t, \action^*_t \sim \dataExp}{\left(\frac{\distStateTime{\learner}{t}(s^*_t)}{\distStateTime{\expert}{t}(s^*_t)}\right) \lipm(\policy(\state^*), \action^* )}
\end{split}
\end{equation}
where 
\begin{equation}
\begin{aligned}
\lipm(\policy(\state^*), \action^* ) = \max_{f \in \F_{t+1}} 
& \expectB{ 
\substack{\action_t \sim \policy(.|\state^*_t) \\ \state_{t+1} \sim P(.|\state^*_t, \action_t)}}{ f(\state_{t+1})}\\ - &\expectB{ 
\substack{\action^*_t \sim \policy(.|\state^*_t) \\ \state^*_{t+1} \sim P(.|\state^*_t, \action^*_t)}}{ f(\state^*_{t+1})}
\end{aligned}
\end{equation}
Comparing with \alcov \eqref{eq:alcov}, we swap out the classification loss $\lclass(.)$ with next state IPM loss $\lipm(.)$. We have the following performance guarantee:
\begin{theorem}[\alcovipm]
\label{thm:alcovipm}
Let $\learner$ be the learnt policy produced by \alcovipm such that $\loss_t(\learner, \dataExp, \distStateTime{\learner}{t}) \leq \epsilon$. 
Let $\hat{r}(s)$ be the density ratio estimate such that $ \expect{\state^*_t \sim \dataExp}{ \hat{r}(s) - \frac{\distStateTime{\learner}{t}(s)}{\distStateTime{\expert}{t}(s)}} \leq \gamma $. Let $A^{\expert}(\state, \action) \leq u \; \forall (s,a)$ be a bound on the (dis)-advantage w.r.t expert. We have
\begin{equation}
    J(\learner) \leq J(\expert) + T (\epsilon + u \gamma)
\end{equation}
\end{theorem}
\begin{proof}
See appendix ~\ref{proof:alcovipm}.
\end{proof}
If the density ratio estimate is perfect, $\gamma = 0$, the bound is $O(T \epsilon)$, i.e. we have no dependence on $u$.

\subsection{Iterative Training Stationary Policies}
\label{sec:approach:iterative}
We now provide a more practical variant of \alice that trains a single stationary policy. We use the same principle in \cite{ross2011reduction} to reduce the chicken-or-egg problem to an iterative no-regret online learning~\cite{shalev2008mind}, for example, via dataset aggregation. 

Algorithm~\ref{alg:framework} provides an outline. It takes in a dataset of cached expert demonstrations $\dataExp$ and a simulator. At iteration $i$, it rolls in the current learner $\learner^i$ via the simulator to construct the learner state distribution $\state_t \sim \distStateTime{\learner^i}{t}$ for any timestep. This in turn is used to compute a dataset of losses across various timesteps. This dataset is then aggregated with previous datasets and a new learner $\policy^{i+1}$ is computed on the whole dataset. For strongly convex losses, Algorithm~\ref{alg:framework} is a no-regret online algorithm, and hence achieves sub-linear regret against the best policy in hindsight. With some regularization\footnote{To achieve no-regret we might either use a strongly convex approximation of the classification loss, or we might require regularization as in online gradient descent \cite{zinkevich2003online} or methods like \textit{weighted majority} \cite{kolter2007dynamic} for arbitrary, non-convex losses.}, we can recover an algorithm $O(T \epsilon)$ with any of the 3 losses.

\begin{algorithm}[!htbp]
   \caption{\algName}
   \label{alg:framework}
\begin{algorithmic}
   \STATE {\bfseries Input:} Cached expert demonstrations $\dataExp = \{(\state_t^*, \action_t^*)\}$, Simulator $\simulator:\pi\rightarrow\distState{\pi}$
   \STATE Initialize dataset $\data \gets \emptyset$
   \STATE Initialize learner $\learner^1$ to any policy in $\policyClass$
   \FOR{$i=1$ {\bfseries to} $N$}
   \STATE Sample states $\state_t \sim \distStateTime{\learner^i}{t}$ by running $\learner^i$ in simulator $\simulator$
   \STATE Compute a dataset of losses $\data_i = \{ \loss_t(\policy, \dataExp, \distStateTime{\learner^i}{t}) \}$.
   \STATE Aggregate datasets: $\data \gets \data \bigcup \data_i$
   \STATE Train learner $\learner^{i+1}$ on aggregated $\data$
   \ENDFOR
   \STATE {\bfseries Return} best $\learner^{i}$ on validation
\end{algorithmic}
\end{algorithm}

\section{Evaluation and Benchmarks}

In this section we take the opportunity to raise broader concerns with the benchmarks commonly used in imitation learning and identify necessary attributes of better benchmarks.

We performed a set experiments using what is now the standard IL approach of training an RL agent to provide a set of demonstrations $\dataExp$, then running IL using $\dataExp$. We report the average reward of the expert dataset and behavioral cloning (BC) in Table~\ref{tab:BC_benchmarks_small}, and share more baselines and results in Appendix \ref{experiments-expanded}. In the stationary policy setting, BC is the initialization point for \alice and a lower bound on our performance. Unfortunately (or fortunately), as a lower bound for our algorithm, BC leaves no room for improvement because the standard benchmarks fall squarely within the realizable ($\epsilon=0$) setting. It is indeed challenging to find settings for which BC does \textit{not} perform well, leaving little room to showcase relative performance gains. That naive BC performs better than or equivalent to sophisticated IL methods on many of these benchmarks is not necessarily a critique of the IL algorithms, which may be useful for combating covariate shift in real-world applications. Rather, it is an acknowledgment of the inadequacy of the benchmark environments themselves to demonstrate the $O(T^2)$ error compounding suffered by BC, commonly observed both theoretically and experimentally \cite{ross2010efficient}.

\begin{table}
\label{tab:BC_benchmarks_small}
\centering
\begin{tabular}{|l|c|c|c|}
    \hline
    \textbf{Environment} & \textbf{Expert} & \textbf{BC} \\ \hline 
    CartPole & $500\pm 0$ & $500\pm 0$ \\ \hline
    Acrobot & $-71.7\pm 11.5$ & $-78.4\pm 14.2$ \\ \hline
    MountainCar & $-99.6\pm10.9$  &  $-107.8\pm 16.4$ \\ \hline
    Hopper & $3554\pm 216$ & $3258\pm 396$ \\ \hline
    Walker2d & $5496\pm 89$ & $5349\pm 634  $ \\ \hline
    HalfCheetah & $4487\pm 164$ & $4605\pm 143$ \\ \hline
    Ant & $4186\pm 1081$ & $3353\pm 1801$ \\ \hline
\end{tabular}

\caption{\textit{Behavioral Cloning} performance on common control benchmarks for 25 expert trajectories, averaged over 7 random seeds. In nearly every case, BC performs within the expert's error margin and often just as well as more sophisticated methods. See appendix for more results and implementation specifics.}
\end{table}


\textbf{Difficulty} - We claim these benchmarks are ``too easy'' in the sense that they do not exhibit the feedback-driven covariate shift and error compounding observed when implementing IL in real world scenarios. This is evidenced by the fact that in many published works \cite{barde2020adversarial}, naive BC outperforms many sophisticated ``state-of-the-art'' IL algorithms, as reported by the authors themselves. In cases where authors report weak performance by BC \cite{ho2016generative,de2019causal}, we have found that repeating those experiments with stronger optimization and different model classes (validating our claim of ``model misspecification'') produces reasonable results.




\textbf{Proposed Benchmark Requirements} - Ideally, IL researchers would develop their own realistic and repeatable set of \textit{IL-centric} benchmarks with the following properties:
\begin{enumerate}
    \item A repeatable sequential decision making environment (\textit{à la} OpenAI Gym, MuJoCo, ...)
    \item A pre-tuned (reward-optimal) expert policy or fixed set of demonstrations, used in common by all researchers
    \item A standard success metric of on-policy expert mismatch: $\expect{\state \sim \distState{\policy}}{\mathbbm{1}_{\policy(\state)\neq\expert(\state)}}$
    \item Scalable difficulty that is nontrivial yet feasible (\textit{i.e.} distinct expert/learner model classes, but adequate expert coverage).
\end{enumerate}
As noted by \cite{zhang2020causal}, imitability or difficulty in IL is closely linked to observability, and partial observability removes realizability and injects often substantial covariate shift into IL problems \cite{wen2020fighting}. However, reducing observability in benchmark problems must be done with care, as it can quickly make the problem so challenging as to become unsolvable even for algorithms such as \DAgger, which achieve the theoretic upper bound in many settings. In pure IL, the upper and lower performance bounds are not \textit{Expert} and \textit{Random}, but carefully optimized \DAgger and BC.

Rather than avoiding the fact that carefully optimized BC performs remarkably well in many existing benchmarks or contriving ways to handicap BC, the community can benefit from a focus on IL-centric benchmarks that exhibit the characteristics of more difficult real world problems.

\section{Discussion}

In this paper we identify the root cause of some classic challenges in imitation learning and present a general framework for addressing them. Specifically, we notice how feedback in sequential decision making tasks causes covariate shift in standard imitation learning, and show how we can correct for that shift under certain assumptions.

Although this \textit{Goldilocks} regime of moderate difficulty problems has been oft been noted in the real world, an active area of future work is to find and develop benchmarks which are not easily solvable by BC yet still feasible. This is also part of a broader call to ourselves and to the broader community of researchers to develop IL-centric benchmarks which are consistent yet flexible.

Within that \textit{Goldilocks} regime we outlined a general IL framework and introduced three specific loss functions to counter and correct covariate shift using density ratio correction and next state moment matching. Those solutions relied on a bounded density-ratio and a notion of one-step recoverability. Single step recoverability is often a very strong requirement, and we would like to consider approaches that can manage with less demanding requirements.

\section*{Acknowledgements}

The authors thank Hal Daume for thoughtful conversation on when causal confounding might play
a significant role in imitation learning and structured prediction and Sergey Levine and Dinesh Jayaraman for discussions on identification of causal models and the role for causality in the compounding-error phenomena.

\bibliography{references}
\bibliographystyle{icml2020}

\newpage
\onecolumn
\appendix
\section{Appendix}

\subsection{Proof of Theorem~\ref{thm:bc_medium} (Behavior Cloning)}
\label{proof:bc}
\begin{proof}
We begin by bounding the expected on-policy dis-advantage of learner w.r.t expert
\begin{equation}
\begin{aligned}
    \expect{\state_t \sim \distStateTime{\policy}{t}}{A^{\expert}(\state_t, \policy(\state_t))}
    &= \expect{\state_t \sim \distStateTime{\expert}{t}}{ \frac{\distState{\policy}(\state_t)}{\distState{\expert}(\state_t)}  A^{\expert}(\state_t, \policy(\state_t))} \\
    &\leq \norm{\frac{\distStateTime{\policy}{t}(.)}{\distStateTime{\expert}{t}(.)}}{\infty} \expect{\state_t \sim \distStateTime{\expert}{t}}{A^{\expert}(\state_t, \policy(\state_t))} \\
    &\leq \norm{\frac{\distStateTime{\policy}{t}(.)}{\distStateTime{\expert}{t}(.)}}{\infty}       
    \norm{ A^{\expert}(., .) }{\infty} 
    \expectB{\substack{\state^*_t \sim \distStateTime{\expert}{t} \\ \action^*_t \sim \expert(.|\state^*_t)}}{ \lclass( \policy(\state^*_t) , \action^*_t ) } \\
    &\leq C u \epsilon \\
\end{aligned}
\end{equation}
where the third inequality follows from the fact that $\lclass(.)$ upper bounds the $0-1$ loss.

Using the Performance Difference Lemma
\begin{equation}
\begin{aligned}
J(\policy) &= J(\expert) + \sum_{t=1}^T \expect{\state_t \sim \distStateTime{\policy}{t}}{A^{\expert}(\state_t, \policy(\state_t))} \\
&\leq J(\expert) + \sum_{t=1}^T C u \epsilon \\
&\leq J(\expert) + T C u \epsilon \\
\end{aligned}
\end{equation}

We can clip the bound above if we assume that costs are bounded $[0,1]$, using Theorem 2.1 from ~\cite{ross2011reduction}
\begin{equation}
    J(\policy) \leq J(\expert) + T^2 \epsilon 
\end{equation}
\end{proof}

\subsection{Proof of Theorem~\ref{thm:alcov} (\alcov)}
\label{proof:alcov}
\begin{proof}
We begin by bounding the expected on-policy dis-advantage of learner w.r.t expert
\begin{equation}
\begin{aligned}
    \expect{\state_t \sim \distStateTime{\policy}{t}}{A^{\expert}(\state_t, \policy(\state_t))}
    & = \expect{\state_t \sim \distStateTime{\expert}{t}}{ \frac{\distState{\policy}(\state_t)}{\distState{\expert}(\state_t)}  A^{\expert}(\state_t, \policy(\state_t))} \\
    & = \expect{\state_t \sim \distStateTime{\expert}{t}}{ \hat{r}(\state_t) A^{\expert}(\state_t, \policy(\state_t))} + \expect{\state_t \sim \distStateTime{\expert}{t}}{ \left(\frac{\distState{\policy}(\state_t)}{\distState{\expert}(\state_t)} - \hat{r}(\state_t)\right)  A^{\expert}(\state_t, \policy(\state_t))} \\
    & \leq u \expectB{ \substack{ \state_t \sim \distStateTime{\expert}{t} \\ \action^*_t \sim \expert(. | \state_t)} }{ \hat{r}(\state_t) \lclass(\policy(\state_t), \action^*_t)} + u \; \expect{\state_t \sim \distStateTime{\expert}{t}}{ \left(\frac{\distState{\policy}(\state_t)}{\distState{\expert}(\state_t)} - \hat{r}(\state_t)\right)} \\
    & \leq u \; \loss(\policy, \dataExp, \distStateTime{\policy}{t}) + u \; \gamma \\
    & \leq u (\epsilon + \gamma) \\
\end{aligned}
\end{equation}

Using the Performance Difference Lemma
\begin{equation}
\begin{aligned}
J(\policy) &= J(\expert) + \sum_{t=1}^T \expect{\state_t \sim \distStateTime{\policy}{t}}{A^{\expert}(\state_t, \policy(\state_t))} \\
&\leq J(\expert) + \sum_{t=1}^T u (\epsilon + \gamma) \\
&\leq J(\expert) + T u (\epsilon + \gamma) \\
\end{aligned}
\end{equation}

\end{proof}

\subsection{Proof of Theorem~\ref{thm:alipm} (\alipm)}
\label{proof:alipm}
\begin{proof}
We begin by bounding the expected on-policy dis-advantage of learner w.r.t expert
\begin{equation}
\begin{aligned}
\label{eq:adv:ipm}
    \expect{\state_t \sim \distStateTime{\policy}{t}}{A^{\expert}(\state_t, \policy(\state_t))} 
    &= \expectB{\substack{\state_t \sim \distStateTime{\policy}{t} \\ \action_t \sim \policy(.|\state_t) \\ \state_{t+1} \sim P(. | \state_t, \action_t)}}{V^*_{t+1}(\state_{t+1})} - \expectB{\substack{\state_t \sim \distStateTime{\policy}{t} \\ \action^*_t \sim \expert_t(.|\state_t) \\ \state_{t+1} \sim P(. | \state_t, \action^*_t)}}{V^*_{t+1}(\state_{t+1})}  \\
    &= \expectB{\substack{\state_t \sim \distStateTime{\policy}{t} \\ \action_t \sim \policy(.|\state_t) \\ \state_{t+1} \sim P(. | \state_t, \action_t)}}{V^*_{t+1}(\state_{t+1})} - \expectB{\substack{\state^*_t \sim \distStateTime{\expert}{t} \\ \action^*_t \sim \expert_t(.|\state^*_t) \\ \state^*_{t+1} \sim P(. | \state^*_t, \action^*_t)}}{V^*_{t+1}(\state_{t+1})}  \\
    &+ \expectB{\substack{\state^*_t \sim \distStateTime{\expert}{t} \\ \action^*_t \sim \expert_t(.|\state^*_t) \\ \state^*_{t+1} \sim P(. | \state^*_t, \action^*_t)}}{V^*_{t+1}(\state_{t+1})} - \expectB{\substack{\state_t \sim \distStateTime{\policy}{t} \\ \action^*_t \sim \expert_t(.|\state_t) \\ \state_{t+1} \sim P(. | \state_t, \action^*_t)}}{V^*_{t+1}(\state_{t+1})}  \\
    &= \expect{\state_{t+1} \sim \distStateTime{\policy}{t+1}}{V^*_{t+1}(\state_{t+1})} - \expect{\state^*_{t+1} \sim \distStateTime{\expert}{t+1}}{V^*_{t+1}(\state^*_{t+1})}  \\
    &+ \expectB{\substack{\state^*_t \sim \distStateTime{\expert}{t} \\ \action^*_t \sim \expert_t(.|\state^*_t) \\ \state^*_{t+1} \sim P(. | \state^*_t, \action^*_t)}}{V^*_{t+1}(\state_{t+1})} - \expectB{\substack{\state_t \sim \distStateTime{\policy}{t} \\ \action^*_t \sim \expert_t(.|\state_t) \\ \state_{t+1} \sim P(. | \state_t, \action^*_t)}}{V^*_{t+1}(\state_{t+1})}  \\
    & \leq \max_{f \in \F_{t+1}} \expect{\state_{t+1} \sim \distStateTime{\policy}{t+1}}{f(\state_{t+1})} - \expect{\state^*_{t+1} \sim \distStateTime{\expert}{t+1}}{f(\state^*_{t+1})}  \\
    &+ \expectB{\substack{\state^*_t \sim \distStateTime{\expert}{t} \\ \action^*_t \sim \expert_t(.|\state^*_t) \\ \state^*_{t+1} \sim P(. | \state^*_t, \action^*_t)}}{V^*_{t+1}(\state_{t+1})} - \expectB{\substack{\state_t \sim \distStateTime{\policy}{t} \\ \action^*_t \sim \expert_t(.|\state_t) \\ \state_{t+1} \sim P(. | \state_t, \action^*_t)}}{V^*_{t+1}(\state_{t+1})}  \\
\end{aligned}
\end{equation}

The second term is a little tricky to bound since the value function $V^*_{t+1}()$ needs to be pulled back from $t+1$ to $t$ via Bellman operator.  To do this, we borrow the definition of \emph{Inherent Bellman Error (IBE)}~\cite{Sun-2019-118906},
\begin{definition}[Inherent Bellman Error] 
\label{def:ibe}
Let the Bellman Operator on a function $g \in \F_{t+1}$ w.r.t the optimal policy $\expert$ be
\begin{equation}
\B^*_t g(\state_t) \triangleq \expectB{\substack{\action^*_t \sim \expert_t(.|\state_t) \\ \state_{t+1} \sim P(. | \state_t, \action^*_t)}}{g(\state_{t+1})},
\end{equation}
i.e., the pull-back of $g$ from $t+1$ to $t$. This pulled back function may not be in the family of function $\B^*_t g(\state_t) \notin \F_t$. 

We define the Inherent Bellman Error $\ibe$ as the worst-case projection error 
\begin{equation}
\ibe = \max_t \left( \max_{g \in \F_{t+1}} \min_{f \in F_t} \norm{f - \B^*_t g}{\infty} \right)
\end{equation}
\end{definition}

Applying Definition.~\ref{def:ibe}
\begin{equation}
\label{eq:apply_ibe}
\begin{aligned}
    \expectB{\substack{\state^*_t \sim \distStateTime{\expert}{t} \\ \action^*_t \sim \expert_t(.|\state^*_t) \\ \state^*_{t+1} \sim P(. | \state^*_t, \action^*_t)}}{V^*_{t+1}(\state_{t+1})} &\leq \max_{f \in \F_t} \expect{\state^*_t \sim \distStateTime{\expert}{t}}{f(\state^*_t)} + \min_{f' \in F_t} \norm{f' - \B^*_t V^*_{t+1}}{\infty} \\
    &\leq \max_{f \in \F_t} \expect{\state^*_t \sim \distStateTime{\expert}{t}}{f(\state^*_t)} + \ibe \\
\end{aligned}
\end{equation}

Using \eqref{eq:apply_ibe} in \eqref{eq:adv:ipm}
\begin{equation}
\begin{aligned}
    \expect{\state_t \sim \distStateTime{\policy}{t}}{A^{\expert}(\state_t, \policy(\state_t))} 
    & \leq \max_{f \in \F_{t+1}} \expect{\state_{t+1} \sim \distStateTime{\policy}{t+1}}{f(\state_{t+1})} - \expect{\state^*_{t+1} \sim \distStateTime{\expert}{t+1}}{f(\state^*_{t+1})}  \\
    &+ \expectB{\substack{\state^*_t \sim \distStateTime{\expert}{t} \\ \action^*_t \sim \expert_t(.|\state^*_t) \\ \state^*_{t+1} \sim P(. | \state^*_t, \action^*_t)}}{V^*_{t+1}(\state_{t+1})} - \expectB{\substack{\state_t \sim \distStateTime{\policy}{t} \\ \action^*_t \sim \expert_t(.|\state_t) \\ \state_{t+1} \sim P(. | \state_t, \action^*_t)}}{V^*_{t+1}(\state_{t+1})}  \\
    & \leq \max_{f \in \F_{t+1}} \expect{\state_{t+1} \sim \distStateTime{\policy}{t+1}}{f(\state_{t+1})} - \expect{\state^*_{t+1} \sim \distStateTime{\expert}{t+1}}{f(\state^*_{t+1})}  \\
    &+ \max_{f \in \F_t} \expect{\state^*_{t} \sim \distStateTime{\expert}{t}}{f(\state^*_{t})} - \expect{\state_{t} \sim \distStateTime{\policy}{t}}{f(\state_{t})} + 2 \ibe \\
    & \leq \loss(\policy, \dataExp, \distStateTime{\policy}{t+1}) + \loss(\policy, \dataExp, \distStateTime{\policy}{t}) + 2 \ibe \\
    & \leq 2 (\epsilon + \ibe) \\
\end{aligned}
\end{equation}

Using the Performance Difference Lemma
\begin{equation}
\begin{aligned}
J(\policy) &= J(\expert) + \sum_{t=1}^T \expect{\state_t \sim \distStateTime{\policy}{t}}{A^{\expert}(\state_t, \policy(\state_t))} \\
&\leq J(\expert) + \sum_{t=1}^T 2 (\epsilon + \ibe) \\
&\leq J(\expert) + 2 T (\epsilon + \ibe) \\
\end{aligned}
\end{equation}

\end{proof}

\subsection{Proof of Theorem~\ref{thm:alcovipm} (\alcovipm)}
\label{proof:alcovipm}
\begin{proof}

We begin by bounding the expected on-policy dis-advantage of learner w.r.t expert
\begin{equation}
\begin{aligned}
    \expect{\state_t \sim \distStateTime{\policy}{t}}{A^{\expert}(\state_t, \policy(\state_t))}
    & = \expect{\state_t \sim \distStateTime{\expert}{t}}{ \frac{\distState{\policy}(\state_t)}{\distState{\expert}(\state_t)}  A^{\expert}(\state_t, \policy(\state_t))} \\
    & = \expect{\state_t \sim \distStateTime{\expert}{t}}{ \hat{r}(\state_t) A^{\expert}(\state_t, \policy(\state_t))} + \expect{\state_t \sim \distStateTime{\expert}{t}}{ \left(\frac{\distState{\policy}(\state_t)}{\distState{\expert}(\state_t)} - \hat{r}(\state_t)\right)  A^{\expert}(\state_t, \policy(\state_t))} \\
    & \leq \expect{\state_t \sim \distStateTime{\expert}{t}}{ \hat{r}(\state_t) A^{\expert}(\state_t, \policy(\state_t))} + u \; \expect{\state_t \sim \distStateTime{\expert}{t}}{ \left(\frac{\distState{\policy}(\state_t)}{\distState{\expert}(\state_t)} - \hat{r}(\state_t)\right)} \\
    & \leq \expect{\state_t \sim \distStateTime{\expert}{t}}{ \hat{r}(\state_t) A^{\expert}(\state_t, \policy(\state_t))} + u \gamma \\
    & \leq \expect{\state_t \sim \distStateTime{\expert}{t}}{ \hat{r}(\state_t) \left( \expectB{\substack{\action_t \sim \policy(. | \state_t) \\ \state_{t+1} \sim P(. | \state_t, \action_t)}}{V^*_{t+1}(\state_{t+1})} - \expectB{\substack{\action^*_t \sim \expert(. | \state_t) \\ \state^*_{t+1} \sim P(. | \state_t, \action^*_t)}}{V^*_{t+1}(\state^*_{t+1})}\right)} + u \gamma \\
    & \leq \expect{\state_t \sim \distStateTime{\expert}{t}}{ \hat{r}(\state_t) \left( \max_{f \in \F_{t+1}}\expectB{\substack{\action_t \sim \policy(. | \state_t) \\ \state_{t+1} \sim P(. | \state_t, \action_t)}}{f(\state_{t+1})} - \expectB{\substack{\action^*_t \sim \expert(. | \state_t) \\ \state^*_{t+1} \sim P(. | \state_t, \action^*_t)}}{f(\state^*_{t+1})}\right)} + u \gamma \\
    & \leq \loss(\policy, \dataExp, \distStateTime{\policy}{t}) + u \gamma \\
    & \leq \epsilon + u \gamma \\
\end{aligned}
\end{equation}

Using the Performance Difference Lemma
\begin{equation}
\begin{aligned}
J(\policy) &= J(\expert) + \sum_{t=1}^T \expect{\state_t \sim \distStateTime{\policy}{t}}{A^{\expert}(\state_t, \policy(\state_t))} \\
&\leq J(\expert) + \sum_{t=1}^T (\epsilon + u \gamma) \\
&\leq J(\expert) + T (\epsilon + u \gamma) \\
\end{aligned}
\end{equation}

\end{proof}

\subsection{Why Recoverability Matters: Various Recoverability Regimes and Performance Bounds}
\label{sec:recoverability}
\begin{figure*}[ht]
    \centering
    \includegraphics[width=0.8\linewidth]{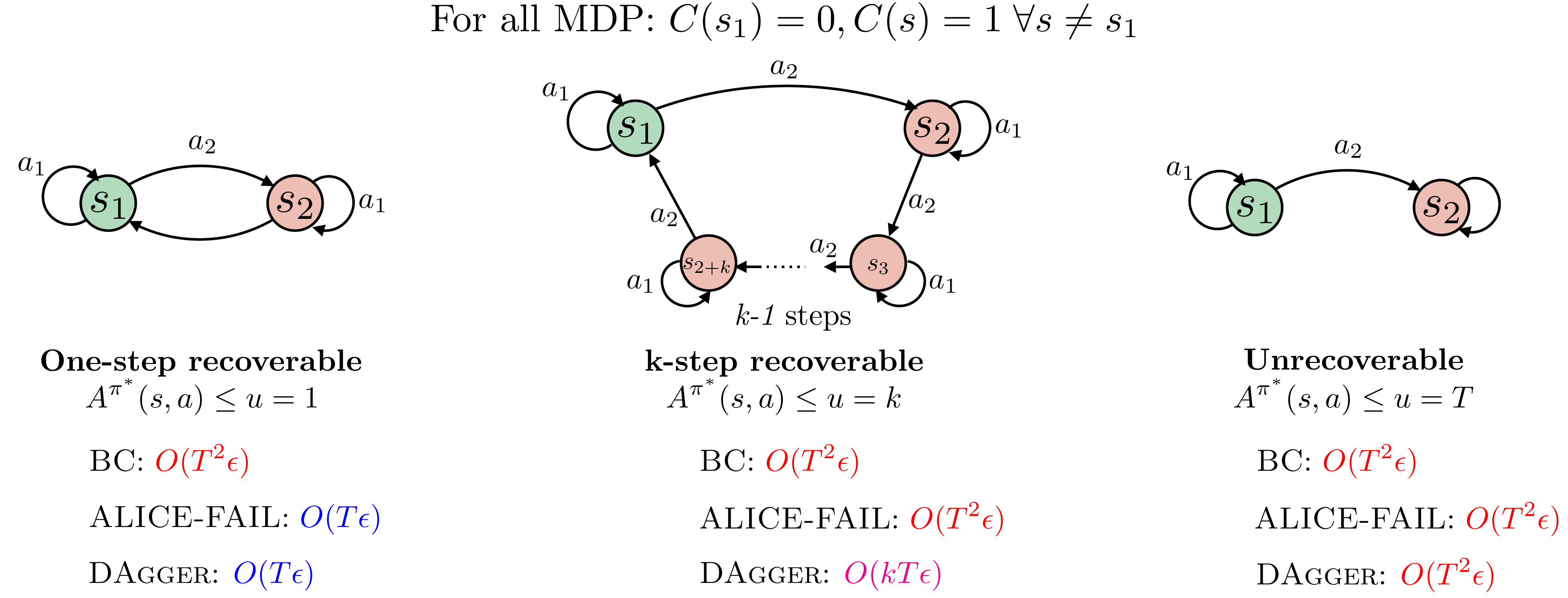}
    \caption{Three different MDPs with varying recoverability regimes. For all MDPs, $C(s_1)=0$ and $C(s) = 1$ for all $s \neq s_1$. The expert deterministic policy is therefore $\expert(s_1)=a_1$ and $\expert(s)=a_2$ for all  $s \neq s_1$.
    Even with one-step recoverability, BC can still result in $O(T^2\epsilon)$ error. For $>1$-step recoverability, even \alipm slides to $O(T^2\epsilon)$, while \DAgger can recover in $k$ steps leading to $O(kT\epsilon)$. For unrecoverable problem, all algorithms can go upto  $O(T^2\epsilon)$. Hence recoverability dictates the lower-bound of how well we can do in the model misspecified regime. }
    \label{fig:recoverability}
\end{figure*}

Here we show that recoverability, in a weaker sense than Def.~\ref{def:recoverability}, is a fundamental requirement for \emph{any IL algorithm}. Consider the class of MDPs defined in Fig.~\ref{fig:recoverability}. They vary in a weaker sense of recoverability, i.e. the upper bound of the advantage $\norm{A^{\expert}(s,a)}{\infty} \leq u$. The first MDP has $u=1$, allowing for one-step recoverability and hence \alipm achieves $O(T\epsilon)$ while BC suffers from compounding error $O(T^2\epsilon)$. The second MDP needs atleast $k-$step to recover, thus both BC and ALICE hit $O(T^2\epsilon)$, while \DAgger can achieve the best bound of $O(kT\epsilon)$. Finally, for unrecoverable MDPs, all algorithms are resigned to $O(T^2\epsilon)$.
\newpage
\subsection{IL Baseline Implementation Specifics}
\label{experiments-expanded}

We use the following settings in training our RL experts (using the RL Baselines Zoo Python package) and IL learners. Learner policy classes consisted of neural networks with relu activation everywhere except final layer and a set number of hidden dimensions (layer 1 dimension, layer 2 dimension, ...). These experiments are all in the stationary setting, where a single policy is trained on all available data and executed for all time-steps.

\begin{table}[h!]
    \centering
    \begin{tabular}{|c|c|}
        \hline
        \textbf{Box2d Discrete OpenAI Gym} & \textbf{MuJoCo Continuous Control} \\
        (CartPole, Acrobot, MountainCar) & (Ant, HalfCheetah, Reacher, Walker, Hopper) \\ \hline
        \underline{RL Expert} & \underline{RL Expert} \\
        Deep Q Network & Soft Actor Critic \\
        100k env. steps & 3M env. steps \\
        exploration fraction = 50\% & learning rate = 0.0003 \\
        final $\epsilon=0.1$ & buffer size = 1M \\ \hline
        \underline{IL Learner} & \underline{IL Learner} \\
        learning rate = 0.001 & learning rate = 0.001 \\
        hidden layer dimensions = (64,) & hidden layer dimensions = (512,512) \\
        100k optimization steps & 20M optimization steps \\ \hline
        
    \end{tabular}
    \caption{Settings used for training RL experts and IL learners}
    \label{tab:my_label}
\end{table}

In some settings, it is natural and useful to augment the current observation $\state_t$ with previous action $\action_{t-1}$, i.e. $\state_t = [\state_t,\action_{t-1}]$. We refer to this addition of single-step action history to BC as "BC+H" and include 

\begin{table}[h!]
\centering
\begin{tabular}{|l|c|c|c|}
    \hline
    \textbf{Environment} & \textbf{Expert} & \textbf{BC} & \textbf{BC+H} \\ \hline 
     CartPole-v1 & $500.0\pm 0.0$ & $500.0\pm 0.0$ & $500.0\pm 0.0$ \\ \hline
     Acrobot-v1  & $-71.7\pm 11.5$ & $-78.4\pm 14.2$ & $-80.0\pm 18.0$ \\ \hline
     MountainCar-v0 & $-99.6\pm 10.9$ & $-107.8\pm 16.4$ & $-105.4\pm 16.6$ \\ \hline
     Ant-v2 & $4185.9\pm 1081.1$ & $3352.9\pm 1800.9$ & $2919.2\pm 1967.9$ \\ \hline
     HalfCheetah-v2 & $4514.6\pm 111.4$ & $4388.4\pm 494.8$ & $4338.1\pm 791.7$ \\ \hline
     Reacher-v2 & $-4.4\pm 1.3$ & $-4.9\pm 2.4$ & $-4.5\pm 1.4$ \\ \hline
     Walker2d & $5496\pm 89$ & $5349\pm 634  $ & $4451\pm 1491$\\ \hline
\end{tabular}
\caption{\textit{Behavioral Cloning} performance on common control benchmarks for 25 expert trajectories. In most cases, BC performs within the expert's error margin and often as well as more sophisticated methods. See appendix for more results and implementation specifics.}
\label{tab:BC_benchmarks}
\end{table}

Although we did not perform experiments for the partially observable setting, please refer to the experiments of \cite{wen2020fighting}. They use a generative adversarial approach to generate an intermediate state representation that removes all information of the previous action. Although their vanilla implementation of BC didn't achieve the strongest results, they include an implementation of their algorithm without the use of an adversary, effectively performing BC with a more principled autoencoder-type model class. Using that improved model class, they show that BC achieves similar reward to DAgger in many of their experiments. 
We also refer the reader to \cite{Sun-2019-118906} for experimental results of Loss \#2 (\FAIL) in the non-stationary policy setting.

\end{document}